\documentclass[twoside,11pt]{article}

\usepackage{jmlr2e}

\usepackage{amsmath}
\usepackage{algorithm}
\usepackage{algorithmic}
\usepackage{url}

\newcommand{\mbf}[1]{\boldsymbol{#1}}

\begin{document}

\title{Guaranteed  inference in topic models}

\author{\name Khoat Than  \email khoattq@soict.hust.edu.vn \\
  \addr Hanoi University of Science and Technology, 1, Dai Co Viet road, Hanoi, Vietnam. 
       \AND
       \name  Tung Doan  \email phongtung\_bmw@yahoo.com \\
         \addr{Hanoi University of Science and Technology, 1, Dai Co Viet road, Hanoi, Vietnam.}
       }

\editor{}

\maketitle

\begin{abstract}
One of the core problems in statistical models is the estimation of a posterior distribution. For topic models, the problem of posterior inference for individual texts is particularly important, especially when dealing with data streams, but is often intractable in the worst case \citep{SontagR11}. As a consequence, existing methods for posterior inference are approximate and do not have any guarantee on neither quality nor convergence rate. In this paper, we introduce a provably fast  algorithm, namely \textit{Online Maximum a Posteriori Estimation (OPE)}, for posterior inference in topic models. OPE has  more attractive properties than existing inference approaches, including theoretical guarantees on quality and fast rate of convergence to a local maximal/stationary point of the inference problem. The discussions  about OPE are very general and hence can be easily employed in a wide range of contexts. Finally, we employ OPE  to  design three  methods for learning Latent Dirichlet Allocation from text streams or large corpora. Extensive experiments demonstrate some superior behaviors of OPE and of our new learning methods. 
\end{abstract}


\section{Introduction}

Latent Dirichlet allocation (LDA) \citep{BNJ03} is the class of Bayesian networks that has gained arguably significant interests. It has found successful applications in a wide range of areas including text modeling \citep{Blei2012introduction}, bioinformatics \citep{PritchardSD2000population,LiuLT10miRNA}, history \citep{Mimno2012historiography}, politics \citep{Grimmer2010Political,GerrishB2012vote}, psychology \citep{Schwartz+2013Personality}, to name a few.

One of the core issues in LDA is the estimation of posterior distributions for individual documents. The research community has been studying many approaches for this estimation problem, such as variational Bayes (VB) \citep{BNJ03}, collapsed variational Bayes (CVB) \citep{TehNW2007collapsed}, CVB0 \citep{Asuncion+2009smoothing}, and collapsed Gibbs sampling (CGS) \citep{GriffithsS2004,MimnoHB12}. Those approaches enable us to easily work with millions of texts 
\citep{MimnoHB12,Hoffman2013SVI,Foulds2013stochastic}. The quality of LDA in practice is determined by the quality of the inference method being employed. However, none of the mentioned methods has a theoretical guarantee on quality or convergence rate. This is a major drawback of existing inference methods. 

Our first contribution in this paper is  the introduction of a provably efficient algorithm, namely \emph{Online Maximum a Posteriori Estimation (OPE)}, for doing posterior inference of topic mixtures in LDA. This inference problem is in fact nonconvex  and is NP-hard \citep{SontagR11,Arora+2016infer}. Our new algorithm  is stochastic in nature and theoretically converges  to a local maximal/stationary point of the inference  problem. We prove that OPE converges at a rate of $O({1/T})$, which surpasses the best rate of existing stochastic algorithms for nonconvex problems \citep{Mairal2013stochasticNonconvex,Ghadimi2013stochasticNonconvex}, where $T$ is the number of iterations. Hence, OPE  overcomes many drawbacks of VB, CVB, CVB0, and CGS. Those properties help OPE to be  preferable  in many contexts, and to provide us real benefits when using OPE in a wide class of probabilistic models.

The topic modeling literature has seen a fast  growing interest in designing large-scale learning algorithms \citep{MimnoHB12,ThanH2012fstm,Broderick2013streaming,Foulds2013stochastic,Patterson2013stochastic,Hoffman2013SVI,ThanD14dolda,SatoN2015SCVB0}. Existing algorithms allow us to easily analyze millions of documents. Those developments are of great significance, even though the posterior estimation is often intractable. Note that the performance of a learning method heavily depends on its core inference subroutine. Therefore, existing large-scale learning methods seem to likely remain some of the drawbacks from VB, CVB, CVB0, and CGS.

Our second contribution in this paper is the introduction of 3 stochastic algorithms for learning  LDA at a large scale: \emph{Online-OPE} which is online learning; \emph{Streaming-OPE} which is streaming learning; and \emph{ML-OPE} which is regularized online learning.\footnote{A slight variant of ML-OPE was shortly presented  in \citep{ThanD14dolda} under a different name of DOLDA.}  These algorithms own the stochastic nature when learning  global variables (topics), and employ OPE as the core for inferring local variables for individual texts, which is also stochastic. They overcome many drawbacks of existing large-scale learning methods owing to the preferable properties of OPE.  From extensive experiments we find that Online-OPE, Streaming-OPE, and ML-OPE often reach  very fast to a high predictiveness level, and are able to consistently increase the  predictiveness of the learned models  as observing more data. In paricular, while Online-OPE surpasses the state-of-the-art methods, ML-OPE often learns tens to thousand times faster than existing methods to reach the same predictiveness level. Therefore, our new methods are efficient tools for analyzing text streams or big collections.

\textsc{Organization:} in the next section we briefly discuss related work. In Section \ref{sec-LDA-infer-theta}, we present  the OPE algorithm for doing posterior inference. We also analyze the convergence property. We further compare OPE with existing inference methods, and discuss how to employ it in other contexts. Section \ref{sec-Dolda} presents three stochastic algorithms for learning LDA from text streams or big text collections. Practical behaviors of  large-scale learning algorithms and OPE will be investigated in Section \ref{sec-Evaluation}. The final section presents some conclusions and discussions. 

\textsc{Notation:}
Throughout the paper, we use the following conventions and notations. Bold faces denote vectors or matrices. $x_i$ denotes the $i^{th}$ element of vector $\mbf{x}$, and $A_{ij}$ denotes the element at row $i$ and column $j$ of matrix $\mbf{A}$.  The unit simplex in the  $n$-dimensional Euclidean space is denoted as $\Delta_n = \{ \mbf{x} \in \mathbb{R}^n: \mbf{x} \ge 0, \sum_{k=1}^{n} x_k = 1 \}$, and its interior is denoted as $\overline{\Delta}_n$. We will work with text collections with $V$ dimensions (dictionary size). Each document $\mbf{d}$ will be represented as a frequency vector, $\mbf{d} = (d_1, ..., d_V)^T$ where $d_j$ represents the frequency of term $j$ in $\mbf{d}$. Denote $n_d$ as the length of $\mbf{d}$, i.e., $n_d = \sum_j d_j$. The inner product of vectors $\mbf{u}$ and $\mbf{v}$ is denoted as $\left<\mbf{u}, \mbf{v}\right>$.

\section{Related work}\label{sec-related-work}
Notable inference methods for probabilistic topic models include VB, CVB, CVB0, and CGS. Except VB \citep{BNJ03}, most other methods originally have been  developed for learning  topic models from data. Fortunately, one can  adapt them to do posterior inference  for individual documents \citep{ThanH15sparsity}. Other good candidates for doing posterior inference include \emph{Concave-Convex procedure} (CCCP) by \cite{Yuille2003cccp}, \emph{Stochastic Majorization-Minimization} (SMM) by \cite{Mairal2013stochasticNonconvex}, \emph{Frank-Wolfe} (FW) \citep{Clarkson2010}, Online Frank-Wolfe (OFW) \citep{Hazan2012OFW}, and  \emph{Thresholded Linear Inverse} (TLI)  which has been newly developed by \cite{Arora+2016infer}. 

Few methods have an explicit theoretical guarantee on inference quality and convergence rate. In spite of being popularly used in topic modeling, we have not seen any theoretical analysis about how fast VB, CVB, CVB0, and CGS do inference for individual documents. One might employ CCCP \citep{Yuille2003cccp} and SMM \citep{Mairal2013stochasticNonconvex} to do inference in topic models. Those two algorithms are guaranteed to converge to a stationary point of the inference problem. However, the convergence rate of CCCP and SMM is unknown for  non-convex  problems which are inherent in LDA and many other models. Each iteration of CCCP has to solve a (non-linear) equation system, which is expensive and non-trivial in many cases. Furthermore, up to now those two methods  have not been investigated rigorously in the topic modeling literature.

It is worth discussing about FW \citep{ThanH15sparsity}, OFW \citep{Hazan2012OFW}, and TLI \citep{Arora+2016infer}, the three methods with theoretical guarantees on quality. FW is a general method for convex programming \citep{Clarkson2010}. \cite{ThanH15sparsity,ThanH2012fstm} find that it can be effectively used to do inference for topic models. OFW is an online version of FW for convex problems whose objective functions come partly in an online fashion. One important property of FW and OFW is that they can converge fast and return sparse solutions. Nonetheless, FW and OFW only work with convex problems, and thus require some special settings/modifications for topic models. On the other hand, TLI has been proposed recently to do exact inference for individual texts. This is the only inference method which is able to recover solutions exactly under some assumptions. TLI requires that a document should be very long, and the topic matrix should have a small condition number. Those conditions might not always be present in practice. Therefore TLI is quite limited and should be improved further.

Two other algorithms for MAP estimation with provable guarantees are \emph{Particle Mirror Decent} (PMD) \citep{Dai2016pmd} and HAMCMC \citep{Simsekli2016stochastic}. Both algorithms base on sampling to estimate a posterior distribution. Therefore they can be used to do posterior inference for topic models. PMC is shown to converge at a rate of $\mathcal{O}(T^{-1/2})$, while HAMCMC converges at a rate of $\mathcal{O}({T}^{-1/3})$ as suggested by \cite{Teh2016consistencySGLD}.\footnote{In fact \cite{Simsekli2016stochastic}  provide an explicit bound on the error as $\mathcal{O}(1/\sum_{t=1}^T \epsilon_t)$, where $\epsilon_t$ defines the step-size of their algorithm. This error bound will go to zero as $T$ goes to infinity. However, the authors did not provided any explicit error bound which directly depends on $T$.} Those are significant developments for Bayesian inference. However, their effectiveness in topic modeling is unclear at the time of writing this article.

In this work, we propose OPE for doing posterior inference. Unlike CCCP and SMM, OPE is guaranteed to converge very fast to a local maximal/stationary point of the inference problem. The convergence rate of OPE is faster than that of PMD and HAMCMC. Each iteration of OPE requires modest arithmetic operations and thus OPE is significantly more efficient than CCCP, SMM, PMD, and HAMCMC.  Having an explicit guarantee helps OPE to  overcome many limitations of VB, CVB, CVB0, and CGS. Further, OPE is so general  that it can be easily employed in a wide range of contexts, including MAP estimation and non-convex optimization. Therefore, OPE overcomes some drawbacks of FW, OFW, and TLI.  Table \ref{table 1: theoretical comparison} presents more details to compare OPE and various inference methods.

\section{Posterior inference with OPE} \label{sec-LDA-infer-theta}

LDA \citep{BNJ03} is a generative model for modeling texts and discrete data. It assumes that a corpus is composed from $K$ topics $\mbf{\beta}_1, ..., \mbf{\beta}_K$, each of which is a sample from a $V$-dimensional Dirichlet distribution, $Dirichlet(\eta)$. A document $\mbf{d}$ arises from the following generative process:
\begin{enumerate}
  \item Draw $\mbf{\theta}_d | \alpha \sim Dirichlet(\alpha)$
  \item For the $n^{th}$ word of $\mbf{d}$:
  \begin{itemize}
    \item[-] draw topic index $z_{dn} | \mbf{\theta}_d \sim Multinomial(\mbf{\theta}_d)$
    \item[-] draw word $w_{dn}| z_{dn}, \mbf{\beta} \sim Multinomial(\mbf{\beta}_{z_{dn}})$.
  \end{itemize}
\end{enumerate}
Each topic mixture $\mbf{\theta}_d = (\theta_{d1}, ..., \theta_{dK})$ represents the contributions of topics to document $\mbf{d}$, while $\beta_{kj}$ shows the contribution of term $j$ to topic $k$. Note that  $\mbf{\theta}_d \in \Delta_K, \mbf{\beta}_k \in \Delta_V, \forall k$. Both $\mbf{\theta}_d$ and $\mbf{z}_d$ are unobserved variables and are local for each document. 

According to \cite{TehNW2007collapsed}, the task of \emph{Bayesian inference (learning)} given a corpus $\mathcal{C} = \{\mbf{d}_1, ..., \mbf{d}_M\}$ is to estimate the posterior distribution $p(\mbf{z, \theta, \beta} | \mathcal{C}, \alpha, \eta)$ over the latent topic indicies $\mbf{z} = \{\mbf{z}_1, ..., \mbf{z}_d\}$, topic mixtures $\mbf{\theta} = \{\mbf{\theta}_1, ..., \mbf{\theta}_M\}$, and topics $\mbf{\beta} = (\mbf{\beta}_1, ..., \mbf{\beta}_K)$. \emph{The problem of posterior inference} for each document $\mbf{d}$, given a model $\{\mbf{\beta}, \alpha\}$, is to estimate the full joint distribution $p(\mbf{z}_d, \mbf{\theta}_d, \mbf{d} | \mbf{\beta}, \alpha)$. Direct estimation of this distribution is intractable. Hence existing approaches uses different schemes. VB, CVB, and CVB0 try to estimate the distribution by maximizing a lower bound of the likelihood $p(\mbf{d} | \mbf{\beta}, \alpha)$, whereas CGS \citep{MimnoHB12} tries to estimate $p(\mbf{z}_d | \mbf{d}, \mbf{\beta}, \alpha)$. For a detailed discussion and comparison of those methods, the reader should refer to \cite{ThanH15sparsity}.

\subsection{MAP inference of topic mixtures}

We now consider the MAP estimation of topic mixture for a given document $\mbf{d}$:
\begin{equation} \label{eq1}
\mbf{\theta}^* = \arg \max_{\mbf{\theta} \in \Delta_K} \Pr(\mbf{\theta}, \mbf{d}|\mbf{\beta},\alpha) = \arg \max_{\mbf{\theta} \in \Delta_K} \Pr(\mbf{d}|\mbf{\theta},\mbf{\beta}) \Pr(\mbf{\theta}|\alpha).
\end{equation} 
\cite{ThanH15sparsity} show that this problem is equivalent to the following one:
\begin{equation} \label{eq2}
\mbf{\theta}^* = \arg \max_{\mbf{\theta} \in \Delta_K} \sum_j d_j \log\sum_{k = 1}^K\theta_k\beta_{kj} + (\alpha - 1)\sum_{k = 1}^K \log\theta_k.
\end{equation} 

\cite{SontagR11} showed that this problem is NP-hard in the worst case when $\alpha < 1$. In the case of $\alpha \ge 1$, one can easily show that the problem (\ref{eq2}) is concave, and therefore it can be solved in polynomial time. Unfortunately, in practice of LDA, the parameter $\alpha$ is often small, says $\alpha < 1$, causing (\ref{eq2}) to be nonconcave. That is the reason for why (\ref{eq2}) is intractable in the worst case.   

We present a novel algorithm (OPE) for doing inference of topic mixtures for documents. The idea of OPE is quite simple. It  solves problem (\ref{eq2}) by iteratively finding a good vertex of $\overline{\Delta}_K = \{\mbf{x} \in \mathbb{R}^K: \sum_{k=1}^{K} x_k = 1, \mbf{x}  \ge \epsilon  > 0 \}$ to improve its solution. A good vertex at each iteration is decided by assessing stochastic approximations to the gradient of the objective function $f(\mbf{\theta})$. When the number of iterations goes to infinity, OPE will approach to a local maximal/stationary point of problem (\ref{eq2}). Details of OPE is presented in Algorithm \ref{alg:OPE-infer}.  

\begin{algorithm}[tp]
   \caption{OPE: Online maximum a posteriori estimation}
   \label{alg:OPE-infer}
\begin{algorithmic}
   \STATE {\bfseries Input: } document $\boldsymbol{d}$, and model $\{\mbf{\beta}, \alpha\}$.
   \STATE {\bfseries Output:} $\boldsymbol{\theta}$  that maximizes 
   $ f(\boldsymbol{\theta}) = \sum_j d_j \log \sum_{k=1}^K \theta_k \beta_{kj}  + (\alpha-1) \sum_{k=1}^{K} \log \theta_k.$
   \STATE Initialize $\mbf{\theta}_1$ arbitrarily in $\overline{\Delta}_K = \{\mbf{x} \in \mathbb{R}^K: \sum_{k=1}^{K} x_k = 1, \mbf{x}  \ge \epsilon > 0 \}$.
   \FOR{ $t = 1, ..., \infty$}
   \STATE Pick  $f_{t}$ uniformly from  $\;\;\; \{\sum_j d_j \log \sum_{k=1}^K \theta_{k} \beta_{kj}; \;\; (\alpha-1) \sum_{k=1}^{K} \log \theta_k \}$
   \STATE $F_{t} := \frac{2}{t} \sum_{h=1}^{t} f_h$
   \STATE $\mbf{e}_t := \arg \max_{\mbf{x} \in \overline{\Delta}_K}  \left<F'_{t}(\boldsymbol{\theta}_{t}), \mbf{x} \right>  \;\;\;\;$ (the vertex of $\overline{\Delta}_K$ that follows the maximal gradient)
   \STATE $\boldsymbol{\theta}_{t +1} := \boldsymbol{\theta}_{t} + (\boldsymbol{e}_{t} - \boldsymbol{\theta}_{t}) /t$
   \ENDFOR
\end{algorithmic}
\end{algorithm}

\subsection{Convergence analysis}
In this section, we prove the convergence of OPE which appears in Theorem \ref{the9}. We need the following observations.


\begin{lemma} \label{lem01}
Let $\{X_1, X_2, ...\}$ be a sequence of  uniformly i.i.d. random variables on $\{-1, 1\}$. (Each $X_i$ is also known as a Rademacher random variable.) The followings hold for the sequence $S_n = X_1 + X_2 + \cdots + X_n$:
\begin{enumerate}
\item $\frac{S_n}{n} \rightarrow 0$ as $n \rightarrow +\infty$.
\item There exist constants $v \in [0, 1)$ and $N_0 > 1$ such that $\forall n \geq N_0, |S_n| \leq n^v$ (equivalently, $\log_n |S_n| \leq v$).
\end{enumerate}
\end{lemma}

\begin{proof}
Let $a_n$ (and $b_n$ respectively) be the number of times that $1$ (and $-1$) appears in the sum $X_1 + X_2 + \cdots + X_n$. So $a_n + b_n =n$ and $S_n = a_n - b_n$.
If $S_n = c n$ for some $c$, then  $a_n = (c+1)n/2, b_n = (1 - c)n/2$. Since $X_i$ is picked uniformly from $\{-1, 1\}$ for every $i$, both $a_n/n$ and $b_n/n$ go to 0.5 as $n \rightarrow +\infty$. This suggests that $c$ goes to 0 as $n \rightarrow +\infty$. Therefore $\frac{S_n}{n} \rightarrow 0$ as $n \rightarrow +\infty$.

We will prove the second result by contrapositive. Assume
\begin{equation} \label{lem1-eq01}
\forall v \in [0, 1), \forall N_0 >1, \exists n \geq N_0 \text{ such that } \log_n |S_n| > v.
\end{equation}
Take an infinite sequence $v_t \in [0, 1)$ such that $v_t \rightarrow 1$ as $t \rightarrow +\infty$. Then statement (\ref{lem1-eq01}) implies that $\forall t \ge 1, \exists n_t$ satisfying
\begin{eqnarray}
\nonumber
\log_{n_1} (n_1 +1) &>& \log_{n_1} |S_{n_1}| > v_1, \\
\label{lem1-eq02}
\log_{n_t} (n_t +1) &>& \log_{n_t} |S_{n_t}| > v_t, \\
\nonumber
n_t &>& n_{t-1} \text{ for } t \ge 2.
\end{eqnarray}

It is easy to see that $\log_{n_t} (n_t +1) \rightarrow 1$ as $t \rightarrow \infty$. Therefore $ \log_{n_t} |S_{n_t}| \rightarrow 1$ as $t \rightarrow \infty$. In other words, $|S_{n_t}| \rightarrow n_t$ as $t \rightarrow \infty$. This is in contrary to the first result. Hence the second result holds.
\end{proof}

\begin{theorem}[Convergence] \label{the9}
	 Consider the objective function $f(\mbf{\theta})$ in problem (\ref{eq2}), given fixed $\mbf{d},\mbf{\beta},\alpha$. For Algorithm~\ref{alg:OPE-infer}, the followings hold
	 \begin{enumerate}
	    \item For any $\mbf{\theta} \in \overline{\Delta}_K$, $F_{t}\left(\mbf{\theta}\right)$ converges  to $f\left(\mbf{\theta}\right)$ as ${t} \rightarrow +\infty$,
	    \item $\mbf{\theta}_{t}$ converges to a local maximal/stationary point $\mbf{\theta}^*$ of $f$ at a rate of $\mathcal{O}(1 / t)$.
	 \end{enumerate}
\end{theorem}

\begin{proof}
Denote $g_1 =  \sum_j d_j \log\sum_{k = 1}^K\theta_k\beta_{kj}$ and  $g_2 = (\alpha - 1)\sum_{k = 1}^K \log\theta_k$, we have $f = g_1 + g_2$. Let $a_{t}$ and $b_{t}$ be the number of times that we have already picked $g_1$ and $g_2$ respectively after $t$ iterations. 
Note that $a_t + b_t = t$. Therefore for any $\mbf{\theta} \in \overline{\Delta}_K$ we have 
\begin{eqnarray}
\label{eq3--OPE}
F_t &=& \frac{2}{t}(a_t g_1 + b_t g_2) \\
\label{eq4--OPE}
F_t - f &=& \frac{a_t - b_t}{t} (g_1 - g_2) = \frac{S_t}{t} (g_1 - g_2) \\
\label{eq5--OPE}
F'_t - f' &=& \frac{a_t - b_t}{t} (g'_1 - g'_2) = \frac{S_t}{t} (g'_1 - g'_2), 
\end{eqnarray}
where we have denoted $S_t = a_t - b_t$. 

For each iteration $t$ of OPE we have to pick uniformly randomly an $f_t$ from $\{g_1, g_2\}$. Now we make a correspondence between $f_t$ and a uniformly random variable $X_t$ on $\{1, -1\}$. This correspondence is a one-to-one mapping. So $S_t$ can be represented as $S_t = X_1 + \cdots + X_t$. Lemma \ref{lem01} shows that $S_t / t \rightarrow 0$ as $t \rightarrow \infty$. Combining this with (\ref{eq4--OPE}) we conclude that the sequence ${F_t}$ converges  to $f$. Also due to (\ref{eq5--OPE}), the   sequence $F'_t$ converges  to $f'$. The convergence holds for any $\mbf{\theta} \in \overline{\Delta}_K$. This proves the first statement of the theorem.

It is easy to see that the sequence $\{\mbf{\theta}_1, \mbf{\theta}_2, ...\}$ converges  to a point $\mbf{\theta}^* \in \overline{\Delta}_K$ at the rate $\mathcal{O}(1/t)$, due to the update of $\boldsymbol{\theta}_{t +1} := \boldsymbol{\theta}_{t} + (\boldsymbol{e}_{t} - \boldsymbol{\theta}_{t}) /t$. We next show that $\mbf{\theta}^*$ is a local maximal/stationary point of $f$. 

Consider 
\begin{eqnarray}
\label{eq7--OPE}
\left\langle F'_t(\mbf{\theta}_t) , \frac{\mbf{e}_{t}- \mbf{\theta}_t}{t} \right\rangle
&=& \left\langle F'_t(\mbf{\theta}_t) - f'(\mbf{\theta}_t), \frac{\mbf{e}_{t}- \mbf{\theta}_t}{t} \right\rangle + \left\langle f'(\mbf{\theta}_t), \frac{\mbf{e}_{t}- \mbf{\theta}_t}{t} \right\rangle \\
\label{eq8--OPE}
&=& \left\langle \frac{S_t}{t}(g'_1(\mbf{\theta}_t) - g'_2(\mbf{\theta}_t)), \frac{\mbf{e}_{t}- \mbf{\theta}_t}{t} \right\rangle + \left\langle f'(\mbf{\theta}_t), \frac{\mbf{e}_{t}- \mbf{\theta}_t}{t} \right\rangle \\
\label{eq9--OPE}
&=& \frac{S_t}{t^2} \left\langle g'_1(\mbf{\theta}_t) - g'_2(\mbf{\theta}_t), {\mbf{e}_{t}- \mbf{\theta}_t} \right\rangle + \left\langle f'(\mbf{\theta}_t), \frac{\mbf{e}_{t}- \mbf{\theta}_t}{t} \right\rangle 
\end{eqnarray}
Note that $g_1, g_2$ are Lipschitz continuous on $\overline{\Delta}_K$. Hence there exists a constant $L$ such that 
\begin{eqnarray}
\left< f'(z), y - z \right> &\le& f(y) - f(z) + L || y- z ||^2, \forall z, y \in \overline{\Delta}_K
\end{eqnarray}
Exploiting this to (\ref{eq9--OPE}) we obtain
\begin{eqnarray}
\nonumber
\left\langle F'_t(\mbf{\theta}_t) , \frac{\mbf{e}_t- \mbf{\theta}_t}{t} \right\rangle 
&\le& \frac{S_t}{t^2} \left\langle g'_1(\mbf{\theta}_t) - g'_2(\mbf{\theta}_t), {\mbf{e}_{t}- \mbf{\theta}_t} \right\rangle + f(\mbf{\theta}_{t+1}) - f(\mbf{\theta}_{t}) + \frac{L}{t^2} ||\mbf{e}_{t}- \mbf{\theta}_t ||^2.
\end{eqnarray}

Since $\mbf{e}_{t}$ and $\mbf{\theta}_t$ belong to $\Delta_K$, the quantity $| \left\langle g'_1(\mbf{\theta}_t) - g'_2(\mbf{\theta}_t), {\mbf{e}_{t}- \mbf{\theta}_t} \right\rangle |$ is bounded above for any $t$. Therefore, there exists a constant $c_2>0$ such that
\begin{eqnarray}
\label{eq11--OPE}
\left\langle F'_t(\mbf{\theta}_t) , \frac{\mbf{e}_t- \mbf{\theta}_t}{t} \right\rangle 
&\le& \frac{c_2 | S_t |}{t^2} + f(\mbf{\theta}_{t+1}) - f(\mbf{\theta}_{t}) + \frac{c_2 L}{t^2}.
\end{eqnarray}
Summing both sides of (\ref{eq11--OPE}) for all $t$ we have
\begin{eqnarray}
\label{eq12--OPE}
\sum_{h=1}^{t} \frac{1}{h} \left\langle F'_h(\mbf{\theta}_h) , \mbf{e}_h- \mbf{\theta}_h \right\rangle 
&\le& \sum_{h=1}^{t} \frac{c_2 | S_h |}{h^2} + f(\mbf{\theta}_{t+1}) - f(\mbf{\theta}_{1}) + \sum_{h=1}^{t}\frac{c_2 L}{h^2}.
\end{eqnarray}
As $t \rightarrow +\infty$ we note that $f(\mbf{\theta}_t) \rightarrow f(\mbf{\theta}^*)$ due to the continuity of $f$. As a result, inequality (\ref{eq12--OPE}) implies 
\begin{eqnarray}
\label{eq13--OPE}
\sum_{h=1}^{+\infty} \frac{1}{h} \left\langle F'_h(\mbf{\theta}_h) , \mbf{e}_h- \mbf{\theta}_h \right\rangle 
&\le& \sum_{h=1}^{+\infty} \frac{c_2 | S_h |}{h^2} + f(\mbf{\theta}^*) - f(\mbf{\theta}_{1}) + \sum_{h=1}^{+\infty}\frac{c_2 L}{h^2}.
\end{eqnarray}

According to Lemma \ref{lem01},  there exist constants $v \in [0, 1)$ and $T_0 > 1$ such that $\forall t \geq T_0, |S_t| \leq t^v$. Therefore  
\begin{eqnarray}
\label{eq13a-OPE}
\sum_{h=1}^{+\infty} \frac{1}{h} \left\langle F'_h(\mbf{\theta}_h) , \mbf{e}_h- \mbf{\theta}_h \right\rangle  &\le& c_2 \sum_{h=1}^{T_0}  \frac{ | S_h |}{h^2} + c_2\sum_{h=T_0+1}^{+\infty} \frac{ h^{v}}{h^2} + f(\mbf{\theta}^*) - f(\mbf{\theta}_{1}) + \sum_{h=1}^{+\infty}\frac{c_2 L}{h^2}.
\end{eqnarray}
Note that the series $\sum_{h=T_0+1}^{+\infty} {h^{v} / h^2}$ converges due to $v \in [0, 1)$, and $\sum_{h=1}^{T_0} {| S_h |}/{h^2}$ is bounded. Further, $\sum_{h=1}^{+\infty} {L}/{h^2} < \infty$. Hence, the right-hand side of (\ref{eq13a-OPE}) is finite. In addition, $\left\langle F'_h(\mbf{\theta}_h) , \mbf{e}_h \right\rangle > \left\langle F'_h(\mbf{\theta}_h) , \mbf{\theta}_h \right\rangle$ for any $h>0$ because of $\mbf{e}_{h} = \arg \max_{\mbf{x} \in \overline{\Delta}_K} \left\langle F'_h(\mbf{\theta}_h) , \mbf{x} \right\rangle$. Therefore we obtain the following
\begin{eqnarray}
\label{eq14--OPE}
0 \le \sum_{h=1}^{+\infty} \frac{1}{h} \left\langle F'_h(\mbf{\theta}_h) , \mbf{e}_h- \mbf{\theta}_h \right\rangle 
&<& \infty.
\end{eqnarray}
In other words, the series $\sum_{h=1}^{+\infty} \frac{1}{h} \left\langle F'_h(\mbf{\theta}_h) , \mbf{e}_h- \mbf{\theta}_h \right\rangle$ converges to a finite constant.

Note that $0 \le \left\langle F'_h(\mbf{\theta}_h) , \mbf{e}_h- \mbf{\theta}_h \right\rangle$ for any $h$. If there exists constant $c_3 >0$ satisfying  $\left\langle F'_h(\mbf{\theta}_h) , \mbf{e}_h- \mbf{\theta}_h \right\rangle \ge c_3$ for an infinite number of $h$'s, then the series $\sum_{h=1}^{+\infty} \frac{1}{h} \left\langle F'_h(\mbf{\theta}_h) , \mbf{e}_h- \mbf{\theta}_h \right\rangle$ could not converge to a finite constant, which is in contrary to (\ref{eq14--OPE}). Therefore, 

\begin{equation}
\label{eq15--OPE}
\left\langle F'_h(\mbf{\theta}_h) , \mbf{e}_h- \mbf{\theta}_h \right\rangle \rightarrow  0 \text{ as } h \rightarrow +\infty.
\end{equation}

In one case, there exists a large $H$ such that $|\mbf{e}_h- \mbf{\theta}_h| \rightarrow 0$ for any $h \ge H$. This suggests one of the vertex of $\overline{\Delta}_K$ is a local maximal point of $f$, which proves the theorem. On the other case, assume that the sequence $\mbf{e}_h- \mbf{\theta}_h$ diverges or converges to a nonzero constant $c_4$. Then (\ref{eq15--OPE}) holds if and only if $F'_h(\mbf{\theta}_h)$  goes to 0 as $h \rightarrow +\infty$. Since $\mbf{\theta}_h \rightarrow \mbf{\theta}^*$, we have 
\begin{equation}
\lim\limits_{h \rightarrow +\infty} || F'_h(\mbf{\theta}_h) || = \lim\limits_{h \rightarrow +\infty} || f'(\mbf{\theta}_h) || = || f'(\mbf{\theta}^*) ||= 0.
\end{equation}
In other words, $\mbf{\theta}^*$ is a stationary point of $f$.

\end{proof}

\subsection{Comparison with existing inference methods}

\begin{table*}[tp]
\caption{Theoretical comparison of 5 inference methods, given a document $\mbf{d}$ and model $\mathcal{M}$ with $K$ topics. MAP denotes maximum a posterior, ELBO denotes maximizing an evidence lower bound on the likelihood. $T$ denotes  the number of iterations. $n_d$ and $\ell_d$ respectively are the number of different terms and number of tokens in $\mbf{d}$.  `-' denotes \emph{`unknown'}. Note that $n_d \le \ell_d$.}
\begin{center}
\begin{scriptsize}
\begin{tabular}{llllll}
\hline
Method & OPE & VB & CVB & CVB0 & CGS \\
\hline
Posterior probability of interest & $\Pr(\mbf{\theta, d} | \mathcal{M})$ & $\Pr(\mbf{\theta, z, d} | \mathcal{M})$ & $\Pr(\mbf{z, d} | \mathcal{M})$ & $\Pr(\mbf{z, d} | \mathcal{M})$ & $\Pr(\mbf{z, d} | \mathcal{M})$ \\
Approach & MAP & ELBO & ELBO & ELBO & Sampling \\
Quality bound & Yes & -& -& -& - \\
Convergence rate & $O(1/T)$ & - & -& -& - \\
Iteration complexity & $O(K. n_d)$ & $O(K. n_d)$ & $O(K. \ell_d)$ & $O(K. \ell_d)$ & $O(K. \ell_d)$ \\
Storage & $O(K)$ & $O(K. n_d)$ & $O(K. \ell_d)$ & $O(K. \ell_d)$ & $O(K. \ell_d)$ \\
$Digamma$ evaluations & 0 & $O(K.n_d)$ & 0 & 0 & $O(K.n_d)$ \\
$Exp$ or $Log$ evaluations  & $O(K.n_d)$ & $O(K.n_d)$ & $O(K. \ell_d)$ & 0 & $O(K.n_d)$ \\
Modification on global variables & No & No & Yes & Yes & No \\
\hline
\end{tabular}
\end{scriptsize}
\end{center}
\label{table 1: theoretical comparison}
\end{table*}

Comparing with other inference approaches (including VB, CVB, CVB0 and CGS), our algorithm has many preferable properties as summarized in Table \ref{table 1: theoretical comparison}. \footnote{A detailed analysis of VB, CVB, CVB0 and CGS can be found in \citep{ThanH15sparsity}.}

\begin{itemize}
\item[-]  OPE has explicitly theoretical guarantees on quality and fast convergence rate. This is the most notable property of OPE, for which existing inference methods often do not have.
\item[-] Its rate of convergence surpasses the best rate of existing stochastic algorithms for non-convex problems \citep{Ghadimi2013stochasticNonconvex,Mairal2013stochasticNonconvex}. Note that OPE can be easily modified to solve more general non-convex problems. Therefore, it is applicable to a wide range of contexts.
\item[-] OPE requires a very modest memory of $O(K)$ for storing temporary solutions and gradients, which is significantly more efficient than VB, CVB, CVB0, and CGS.
\item[-] Each iteration of OPE requires $O(Kn_d)$ computations for computing gradients and updating solutions. This is much more efficient than VB, CVB, CVB0, and CGS in practice.
\item[-] Unlike CVB and CVB0, OPE does not change the global variables when doing inference for individual documents. Hence OPE embarrasingly enables parallel inference, and is more beneficial than CVB and CVB0. 
\end{itemize}

\subsection{Extension to MAP estimation and non-convex optimization}

It is worth realizing that the employment of OPE for other contexts is straightforward. The main step of using OPE is to formulate the problem of interest to be maximization of a function  of the form $f(x) = g_1(x) + g_2(x)$. In the followings, we demonstrate this main step in two  problems which appear in a wide range of contexts.

\emph{The MAP estimation problem} in many probabilistic models is the task of finding an 
\[x^* = \arg \max_x \Pr(x | D) = \arg \max_x \Pr(D | x) \Pr(x)/ \Pr(D),\] 
where $\Pr(D | x)$ denotes the likelihood of an observed variable $D$, $\Pr(x)$ denotes the prior of the hidden variable $x$, and $\Pr(D)$ denotes the marginal probability of $D$. Note that 
\[x^* =\arg \max_x \Pr(D | x) \Pr(x) = \arg \max_x [\log \Pr(D | x) + \log \Pr(x)].\]
If the densities of the distributions over $x$ and $D$ can be described by some analytic functions,\footnote{The exponential family of  distributions is an example.} then the MAP estimation problem turns out to be maximization of $f(x) = g_1(x) + g_2(x)$, where $g_1(x) =\log \Pr(D | x), g_2(x) =\log \Pr(x)$. 

Now consider a general \emph{non-convex optimization} problem $x^* =\arg \max_x f(x)$. Theorem~1 by \cite{Yuille2003cccp} shows that we can always decompose $f(x)$ into the sum of a convex function and a concave function, provided that $f(x)$ has bounded Hessian. This is one way to decompose $f$ into the sum of two functions. We can use many other ways to make a decomposition of $f = g_1 + g_2$ and then employ OPE, because the convergence proof of OPE does not require $g_1$ and $g_2$ to be concave/convex.

The analysis above demonstrates that OPE can be easily employed in a wide range of contexts, including posterior estimation and non-convex optimization. One may  need to suitably modify the domain of $\mbf{\theta}$, and hence the step of finding $\mbf{e}_t$ will be a linear program which can be solved efficiently. Comparing with non-linear steps in CCCP \citep{Yuille2003cccp} and SMM \citep{Mairal2013stochasticNonconvex}, OPE could be much more efficient. In this paper, we do not try to make a rigorous investigation of OPE in those contexts, and leave it open for future research.

\section{Stochastic algorithms for learning LDA} \label{sec-Dolda}

We have seen many attractive properties of OPE that other methods do not have. We further show in this section the  simplicity of using OPE for designing fast learning algorithms for topic models. More specifically, we design 3  algorithms: \emph{Online-OPE} which learns LDA from large corpora in an online fashion, \emph{Streaming-OPE} which learns LDA from data streams, and \emph{ML-OPE} which enables us to learn LDA from either large corpora or data streams. These algorithms employ OPE to do MAP inference for individual documents, and the online scheme \citep{Bottou1998stochastic,Hoffman2013SVI} or streaming scheme \citep{Broderick2013streaming} to infer  global variables (topics). Hence, the stochastic nature appears in both local and global inference phases. Note that the MAP inference of local variables by OPE has theoretical guarantees on quality and convergence rate. Such a property might help the new large-scale learning algorithms be more attractive than existing ones, which base on VB, CVB, CVB0, and CGS.

\subsection{Regularized online learning}

Given a corpus $\mathcal{C}$ (with finite or infinite number of documents) and $\alpha > 0$, we will estimate the topics $\mbf{\beta}_1, ...,\mbf{\beta}_K$ that maximize
\begin{equation} \label{eq13}
\begin{split}
\mathcal{L}(\mbf{\beta}) &= \sum_{\mbf{d} \in \mathcal{C}} \log \Pr \left(\mbf{\theta}_d, \mbf{d} | \mbf{\beta},\alpha\right)\\
&= \sum_{\mbf{d} \in \mathcal{C}} \left(\sum_j d_j \log\sum_{k = 1}^K\theta_{dk}\beta_{kj} + (\alpha - 1)\sum_{k = 1}^K \log\theta_{dk}\right) + constant.
\end{split}
\end{equation}

To solve this problem, we use the online learning scheme by \cite{Bottou1998stochastic}. More specifically, we repeat the following steps:
\emph{\begin{itemize}
	\item[-] Sample a subset $\mathcal{C}_t$ of documents from $\mathcal{C}$. Infer the local variables ($\mbf{\theta}_d$) for each document $\mbf{d} \in \mathcal{C}_t$, given the global variable $\mbf{\beta}^{t - 1}$ in the last step.
	\item[-] Form an intermediate global variable $\hat{\mbf{\beta}}^t$ for $\mathcal{C}_t$.
	\item[-] Update the global variable to be a weighted average of $\hat{\mbf{\beta}}^t$ and $\mbf{\beta}^{t - 1}$.
\end{itemize}}
Details of this learning algorithm is presented in Algorithm \ref{alg:ML-OPE}, where we have used  the same arguments as \cite{ThanH2012fstm} to  update the intermediate topics $\hat{\mbf{\beta}}^t$ from $\mathcal{C}_t$:
\begin{equation} \label{eq14}
\hat{\beta}^t_{kj} \propto \sum_{\mbf{d} \in \mathcal{C}_t} d_j \theta_{dk}.
\end{equation}

Note that in Algorithm \ref{alg:ML-OPE} the step-size  $\rho_t = \left(t + \tau\right)^{-\kappa}$  satisfies two conditions: $\sum_{t = 1}^{\infty} \rho_t = \infty$ and $\sum_{t = 1}^{\infty} \rho_t^2 < \infty$. These conditions are to assure that the learning algorithm will converge to a stationary point \citep{Bottou1998stochastic}. $\kappa \in (0.5, 1]$ is the forgeting rate, the higher the lesser the algorithm weighs the role of new data. 

\begin{algorithm}[tp]
	\caption{\textsf{ML-OPE} for learning LDA from massive/streaming data}
	\label{alg:ML-OPE}
	\begin{algorithmic}
		\STATE {\bfseries Input: } data sequence, $K, \alpha, \tau > 0, \kappa \in (0.5, 1]$  
		\STATE {\bfseries Output: } $\mbf{\beta}$		
		\STATE Initialize $\mbf{\beta}^0$ randomly in $\Delta_V$
		\FOR{ $t = 1, ..., \infty$}
		\STATE Pick a set $\mathcal{C}_t$ of documents
		\STATE Do inference by OPE for each $\mbf{d} \in \mathcal{C}_t$  to get  $\mbf{\theta}_d$, given $\mbf{\beta}^{t-1}$
		\STATE Compute intermediate topics $\hat{\mbf{\beta}}^t$ as: 
		\begin{equation}
			\label{eq-ML-OPE-1}
			\hat{\beta}_{kj}^t  \propto  \sum_{\mbf{d} \in \mathcal{C}_t} d_j \theta_{dk}
		\end{equation}
		\STATE Set step-size: $\rho_t = \left(t + \tau\right)^{-\kappa}$
		\STATE Update topics: $\mbf{\beta}^t := \left(1 - \rho_t\right)\mbf{\beta}^{t-1} + \rho_t\hat{\mbf{\beta}}^t$
		\ENDFOR 
	\end{algorithmic}
\end{algorithm}

It is worth discussing some fundamental differences between ML-OPE and existing online/streaming methods. First, we need not to know a priori how many documents to be processed. Hence, ML-OPE can deal well with streaming/online environments in a realistical way. Second, ML-OPE learns topics ($\mbf{\beta}$) directly instead of learning the parameter ($\mbf{\lambda}$) of a variarional distribution over topics as in SVI \citep{Hoffman2013SVI}, in SSU \citep{Broderick2013streaming}, and in the hybrid method by \cite{MimnoHB12}. While $\mbf{\beta}$ are regularized to be in the unit simplex $\Delta_V$, $\mbf{\lambda}$ can grow arbitrarily. The uncontrolled growth of $\mbf{\lambda}$ might potentially cause overfitting as sufficiently many documents are processed. In contrast, the regularization on $\mbf{\beta}$ helps ML-OPE to avoid overfitting and generalize better.

\subsection{Online and streaming learning}
In the existing literature of topic modeling, most methods for posterior inference try to estimate $\Pr(\mbf{\theta, z, d} | \mathcal{M})$ or $\Pr(\mbf{z, d} | \mathcal{M})$ given a model $\mathcal{M}$ and document $\mbf{d}$. Therefore, existing large-scale learning algorithms for topic models often base on those probabilities. Some examples include SVI \citep{Hoffman2013SVI}, SSU \citep{Broderick2013streaming}, SCVB0 \citep{Foulds2013stochastic,SatoN2015SCVB0}, Hybrid sampling and SVI \citep{MimnoHB12}, Population-VB \citep{McInerney2015population}.

Different with other approaches, OPE directly infers $\mbf{\theta}$ by maximizing the joint probability $\Pr(\mbf{\theta, d} | \mathcal{M})$. Following the same arguments with \cite{ThanH15sparsity}, one can easily exploit OPE to design new online and streaming algorithms for learning LDA. Online-OPE in Algorithm~\ref{alg:-online-OPE} and Streaming-OPE in Algorithm~\ref{alg:-stream-OPE} are two exploitations of OPE. It is worth  noting that Online-OPE and Streaming-OPE are hybrid combinations of OPE with SVI, which are similar in manner to the algorithm by \cite{MimnoHB12}. Further, Streaming-OPE and ML-OPE do not need to know a priori the number of documents to be processed in the future, and hence are suitable to work in a real streaming environment.

\begin{algorithm}[tp]
   \caption{\textsf{Online-OPE} for learning LDA from massive data}
   \label{alg:-online-OPE}
\begin{algorithmic}
   \STATE {\bfseries Input: } training data $\mathcal{C}$ with $D$ documents, $K, \alpha, \eta, \tau > 0, \kappa \in (0.5, 1]$  
   \STATE {\bfseries Output:} $\mbf{\lambda}$
   \STATE Initialize $\mbf{\lambda}^0$ randomly
   \FOR{ $t = 1, ..., \infty$}
	   \STATE Sample a set $\mathcal{C}_t$ consisting of $S$ documents. 
	   \STATE Use Algorithm~\ref{alg:OPE-infer} to do posterior inference for each document $\mbf{d} \in \mathcal{C}_t$, given the global variable $\mbf{\beta}^{t-1} \propto \mbf{\lambda}^{t - 1}$ in the last step, to get topic mixture $\mbf{\theta}_d$. Then compute $\mbf{\phi}_d$ as 
		   \begin{equation} 
		   \phi_{djk} \propto \theta_{dk}  \beta_{kj}.
		   \end{equation}
	   \STATE For each $k \in \{1, 2, ..., K\}$, form an intermediate global variable $\hat{\mbf{\lambda}}_k$ for $\mathcal{C}_t$ by
	   		\begin{equation}
	   		  \hat{\lambda}_{kj} = \eta + \frac{D}{S} \sum_{\mbf{d} \in \mathcal{C}_t} d_j \phi_{djk}.
	   		\end{equation}
	   	\STATE Update the global variable by, where $\rho_t = (t + \tau)^{-\kappa}$, 
   			\begin{equation}
   			  \mbf{\lambda}^{t} := (1-\rho_t) \mbf{\lambda}^{t-1} + \rho_t \hat{\mbf{\lambda}}.
   			\end{equation}
   	\ENDFOR
\end{algorithmic}
\end{algorithm}

\begin{algorithm}[tp]
   \caption{\textsf{Streaming-OPE} for learning LDA from massive/streaming data}
   \label{alg:-stream-OPE}
\begin{algorithmic}
   \STATE {\bfseries Input: } data sequence, $K, \alpha$
   \STATE {\bfseries Output:} $\mbf{\lambda}$
   \STATE Initialize $\mbf{\lambda}^0$ randomly
   \FOR{ $t = 1, ..., \infty$}
	   \STATE Sample a set $\mathcal{C}_t$ of documents. 
	   \STATE Use Algorithm~\ref{alg:OPE-infer} to do posterior inference for each document $\mbf{d} \in \mathcal{C}_t$, given the global variable $\mbf{\beta}^{t-1} \propto \mbf{\lambda}^{t - 1}$ in the last step, to get topic mixture $\mbf{\theta}_d$. Then compute $\mbf{\phi}_d$ as 
		   \begin{equation} 
		   \phi_{djk} \propto \theta_{dk}  \beta_{kj}.
		   \end{equation}
	   \STATE For each $k \in \{1, 2, ..., K\}$, compute  sufficient statistics $\hat{\mbf{\lambda}}_k$ for $\mathcal{C}_t$ by
	   		\begin{equation}
	   		  \hat{\lambda}_{kj} = \sum_{\mbf{d} \in \mathcal{C}_t}  d_j \phi_{djk}.
	   		\end{equation}
	   	\STATE Update the global variable by
   			\begin{equation}
   			  \mbf{\lambda}^{t} := \mbf{\lambda}^{t-1} +  \hat{\mbf{\lambda}}.
   			\end{equation}
   	\ENDFOR
\end{algorithmic}
\end{algorithm}

\section{Empirical evaluation} \label{sec-Evaluation}

This section is devoted to investigating  practical behaviors of OPE,  and how useful it is when OPE is employed to design new algorithms for learning topic models at large scales. To this end, we take the following  methods, datasets, and performance measures into investigation.

\textsc{Inference methods:}
\begin{itemize}
\item[-] \emph{Online MAP estimation} (OPE).
\item[-] \emph{Variational Bayes} (VB) \citep{BNJ03}.
\item[-] \emph{Collapsed variational Bayes} (CVB0) \citep{Asuncion+2009smoothing}. 
\item[-] \emph{Collapsed Gibbs sampling} (CGS) \citep{MimnoHB12}.
\end{itemize}

CVB0 and CGS have been observing to work best by several previous studies \citep{Asuncion+2009smoothing,MimnoHB12,Foulds2013stochastic,GaoSWYZ15,SatoN2015SCVB0}. Therefore they can be considered as the state-of-the-art inference methods.

\textsc{Large-scale learning methods:}
\begin{itemize}
\item[-] Our new algorithms:   \emph{ML-OPE}, \emph{Online-OPE},  \emph{Streaming-OPE}
\item[-] \emph{Online-CGS} by \cite{MimnoHB12}
\item[-] \emph{Online-CVB0} by \cite{Foulds2013stochastic}
\item[-] \emph{Online-VB} by \cite{Hoffman2013SVI}, which is often known as SVI
\item[-] \emph{Streaming-VB} by \cite{Broderick2013streaming} with original name to be SSU
\end{itemize}

Online-CGS \citep{MimnoHB12} is a hybrid algorithm, for which CGS is used to estimate the distribution of local variables ($\mbf{z}$) in a document, and VB is used to estimate the distribution of global variables ($\mbf{\lambda}$). Online-CVB0 \citep{Foulds2013stochastic} is an online version of the batch algorithm by \cite{Asuncion+2009smoothing}, where local inference for a document is done by CVB0. Online-VB \citep{Hoffman2013SVI} and Streaming-VB \citep{Broderick2013streaming} are two stochastic algorithms for which local inference for a document is done by VB. To avoid possible bias in our investigation, we wrote 6 methods by Python in a unified framework with our best efforts, and Online-VB was taken from \url{http://www.cs.princeton.edu/~blei/downloads/onlineldavb.tar}.

\textsc{Data for experiments:} The following three large corpora were used in our experiments. \textit{Pubmed} consists of 8.2 millions of medical articles from the pubmed central; \textit{New York Times} consists of 300,000 news;\footnote{
The data were retrieved from \url{http://archive.ics.uci.edu/ml/datasets/}} and \emph{Tweet} consists of nearly 1.5 millions tweets.\footnote{We crawled tweets from Twitter (\url{http://twitter.com/}) with 69 hashtags containing various kinds of topics. Each document is the text content of a tweet. Then all tweets went through a preprocessing procedure including tokenizing, stemming, removing stopwords, removing low-frequency words (appear in less than 3 documents), and removing extremely short tweets (less than 3 words). Details of this dataset can be found at \cite{Mai2016hdp}.} The vocabulary size ($V$) of New York Times and Pubmed is more than 110,000, while that of Tweet is more than 89,000. It is worth noting that the first two datasets contain long documents, while Tweet contains very short tweets. The shortness of texts poses various difficulties \citep{Tang2014understandingLDA,Arora+2016infer,Mai2016hdp}. Therefore the usage of both long and short texts in our investigation would show more insights into performance of different methods. For each corpus we set aside randomly 1000 documents for testing, and used the remaining for learning.

\textsc{Parameter settings:} 
\begin{itemize}
\item[-] \emph{Model parameters:} $K=100, \alpha = 1/K, \eta = 1/K$. Such a choice of $(\alpha, \eta)$ has been observed to work well in many previous studies \citep{GriffithsS2004,Hoffman2013SVI,Broderick2013streaming,Foulds2013stochastic}.
\item[-] \emph{Inference parameters:} at most 50 iterations were allowed  for OPE and VB to do inference. We terminated VB if the relative improvement of the lower bound on likelihood is not better than $10^{-4}$. 50 samples were used in CGS for which the first 25 were discarded and the remaining were used to approximate the posterior distribution. 50 iterations were used to do inference in CVB0, in which the first 25 iterations were burned in. Those number of samples/iterations are often enough to get a good inference solution, according to \cite{MimnoHB12,Foulds2013stochastic}. 
\item[-] \emph{Learning parameters:} minibatch size $S = | \mathcal{C}_t | =5000$, $\kappa = 0.9, \tau=1$. This choice of learning parameters has been found to result in competitive performance of Online-VB \citep{Hoffman2013SVI} and Online-CVB0 \citep{Foulds2013stochastic}. Therefore it was used in our investigation to avoid possible bias. We used default values for some other parameters in Online-CVB0.
\end{itemize}

\textsc{Performance measures:} We used \textit{NPMI} and \textit{Predictive Probability} to evaluate the learning methods. NPMI \citep{Lau2014npmi} measures the semantic quality of individual topics. From extensive experiments, \cite{Lau2014npmi} found that NPMI agrees well with human evaluation on the interpretability of topic models. Predictive probability \citep{Hoffman2013SVI} measures the predictiveness and generalization of a model  to new data. Detailed descriptions of these measures are presented in Appendix \ref{appendix--perp}.

\subsection{Performance of learning methods}

We first investigate the performance of the learning methods when spending more time on learning from data. Figure~\ref{fig-OPE-perp-npmi-time} shows how good they are. We observe that OPE-based methods and Online-CGS are among the fastest methods, while Online-CVB0, Online-VB and Streaming-VB performed very slowly. Remember from Table \ref{table 1: theoretical comparison} that VB requires various evaluations of expensive functions (e.g., digamma, exp, log), while CVB0 needs to update a large number of statistics which associate with each token in a document. That is the reason for the slow performance of Online-CVB0 and VB-based methods. In contrast, each iteration of OPE and CGS is very efficient. Further, OPE can converge very fast to a good approximate solution of the inference problem. Those reasons explain why OPE-based methods and Online-CGS learned significantly more efficiently than the others.

\begin{figure}[!ht]
\centering
\includegraphics[width=\textwidth]{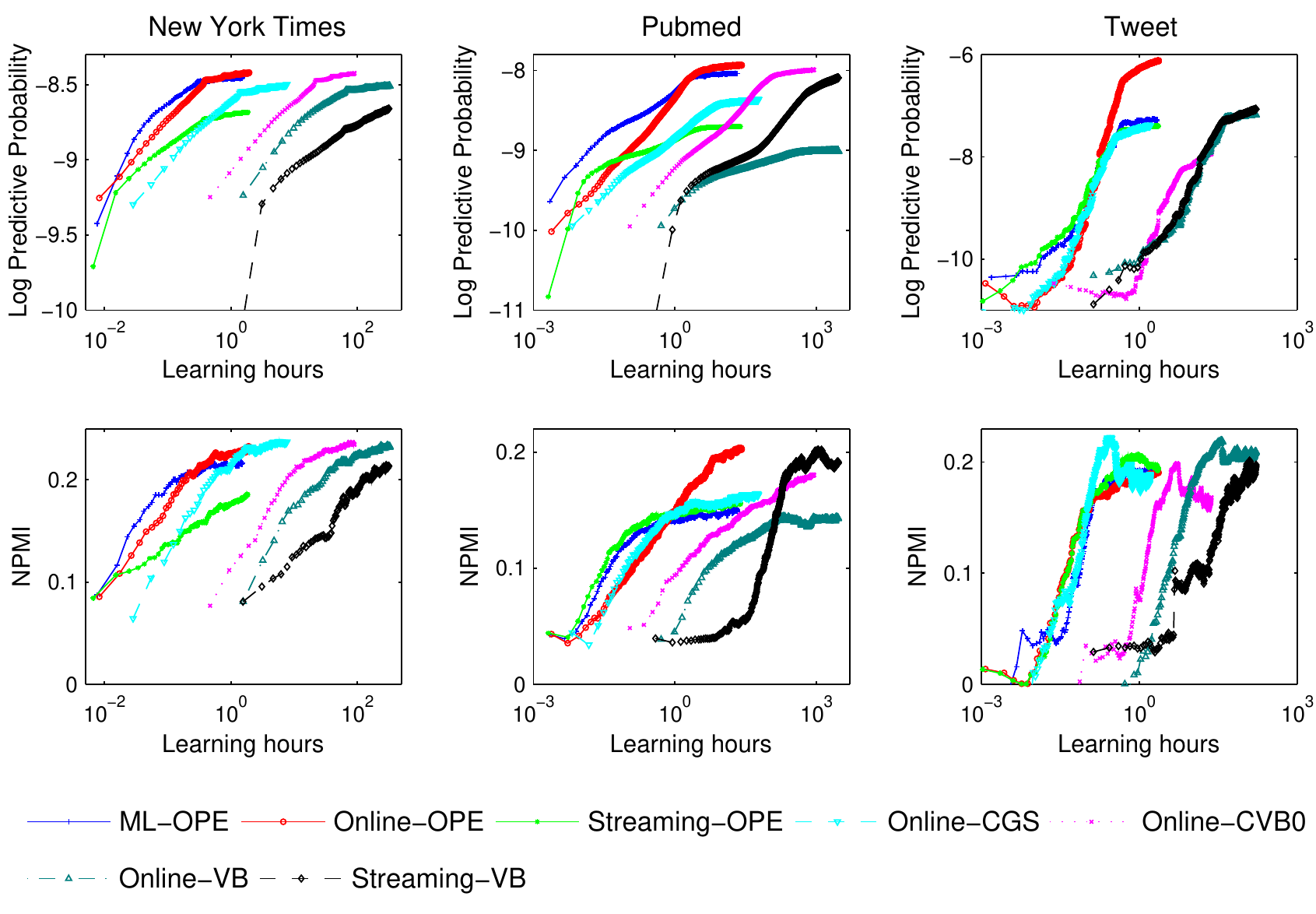}
\caption{Predictiveness (log predictive probability) and semantic quality (NPMI) of the models learned by different methods as spending more time on learning. Higher is better. Note that ML-OPE and Online-OPE often reach state-of-the-art performance. To reach the same predictiveness level, ML-OPE and Online-OPE work many times faster than Online-CGS, hundred times faster than Online-CVB0, and thousand times faster than both Online-VB and Streaming-VB.}
\label{fig-OPE-perp-npmi-time}
\end{figure} 

In terms of predictiveness, Streaming-OPE seems to perform worst, while 6 other methods often perform well. It is worth observing that while ML-OPE consistently reached state-of-the-art performance, Online-OPE surpassed all other methods for three datasets. Online-OPE even outperformed the others with a significant margin in Tweet, despite that such a dataset contains very short documents. ML-OPE and Online-OPE can quickly reach to a high predictiveness level, while Online-VB, Online-CGS, Online-CVB0, and Streaming-VB need substantially more time. Online-VB, Online-CGS, Online-CVB0, and Streaming-VB can perform considerably well as provided long time for leanring. Note that those four methods did not consistently perform well. For example, Online-VB worked on Pubmed worse than other methods, Online-CVB0 performed on Tweet worse than the others. The consistent superior performance of Online-OPE and ML-OPE might be due to the fact that the inferred solutions by OPE are provably good as guaranteed in Theorem \ref{the9}. However, the goodness of OPE seems not to be inherited well in Streaming-OPE.

In terms of semantic quality measured by NPMI, Figure \ref{fig-OPE-perp-npmi-time} shows the results from our experiments. It is easy to observe that Online-OPE often learns models with a good semantic quality, and is often among the top performers. Online-OPE and Online-CGS can return models with a very high quality after a short learning time. In contrast, Streaming-OPE and Online-CVB0 were among the worst methods. Interestingly, Streaming-VB often reach state-of-the-art performance when provided long time for learning, although many of its initial steps are quite bad. Comparing with predictiveness, NMPI did not consistently increase as the learning methods were allowed more time to learn. NPMI from long texts seems to be more stable than that for short texts, suggeting that learning from long texts often gets more coherent models than learning from short texts. In our experience with short texts, most methods often return models with many incoherent and noisy topics for some initial steps. 

\begin{figure}[!ht]
\centering
\includegraphics[width=\textwidth]{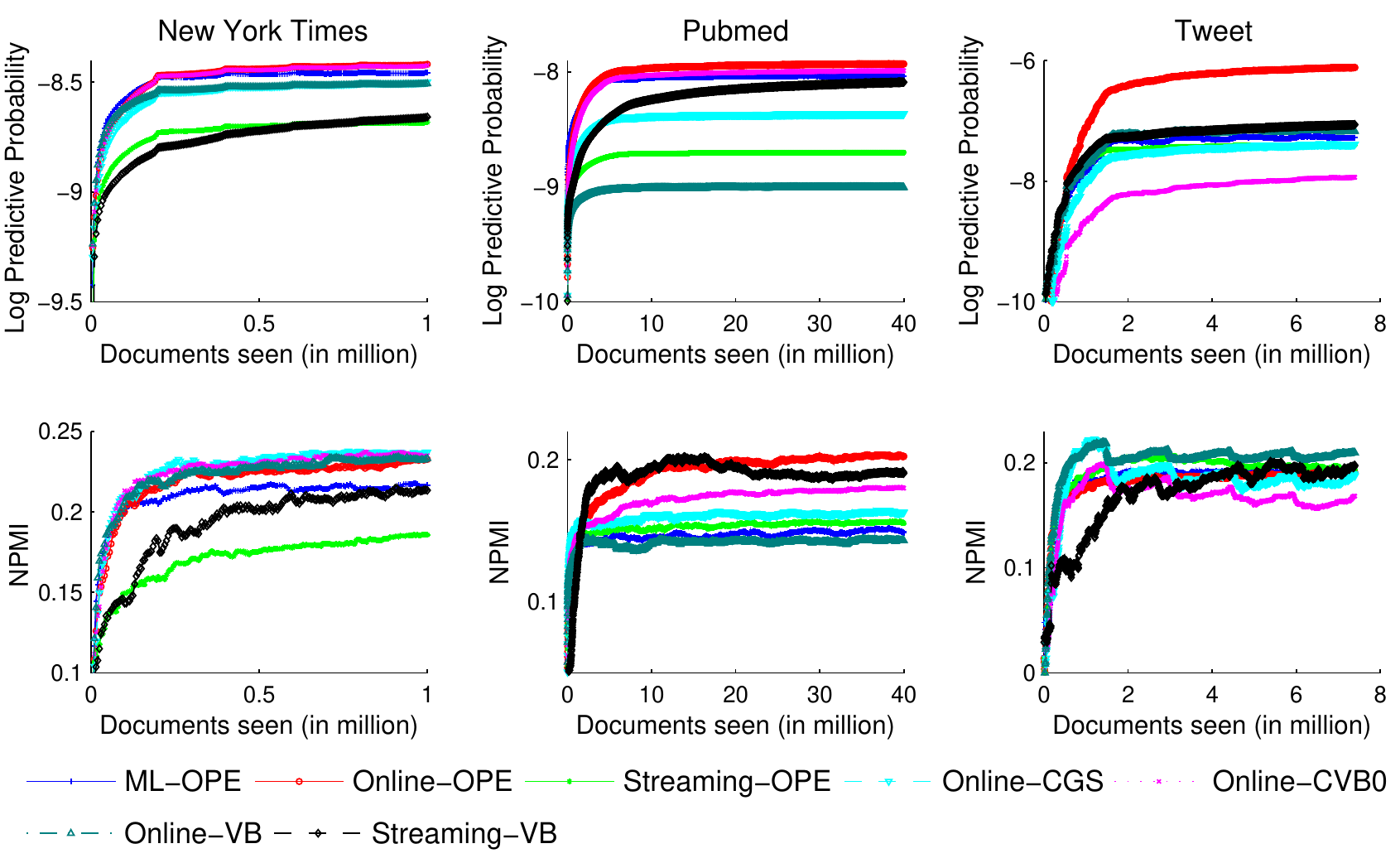}
\caption{Performance of  different learning methods as seeing more documents. Higher is better. Online-OPE often surpasses all other methods, while ML-OPE performes comparably with existing methods in terms of predictiveness.}
\label{fig-OPE-perp-npmi}
\end{figure}

Figure \ref{fig-OPE-perp-npmi} shows another perspective of  performance, where 7 learning methods were fed more documents for learning and were allowed to pass a dataset many times. We oberve that most methods can reach to a high predictiveness and semantic quality after reading some tens of thousand documents. All methods improve their predictiveness as learning from more documents. The first pass over a dataset often helps the learning methods to increase predictiveness drastically. However, more passes over a dataset are able to help improve predictiveness slightly. We find that Streaming-VB often needs many passes over a dataset in order to reach comparable predictiveness. It is easy to see that Online-OPE surpassed all other methods as learning from more documents, while ML-OPE performed comparably. Though being fed more documents in Tweet, existing methods were very hard to reach the same performance of Online-OPE in terms of predictiveness.

Those investigations suggest that ML-OPE and Online-OPE can perform comparably or even significantly better than existing state-of-the-art methods for learning LDA at large scales. This further demonstrates another benefit of OPE for topic modeling.

\textit{A sensitivity analysis:} We have seen an impressive performance of some OPE-based methods. We find that ML-OPE and Online-OPE can be potentially useful in practice of large-scale modeling. We next help them more usable by analyizing sensitivity with their parameters as below.

\begin{figure}[!ht]
\centering
\includegraphics[width=\textwidth]{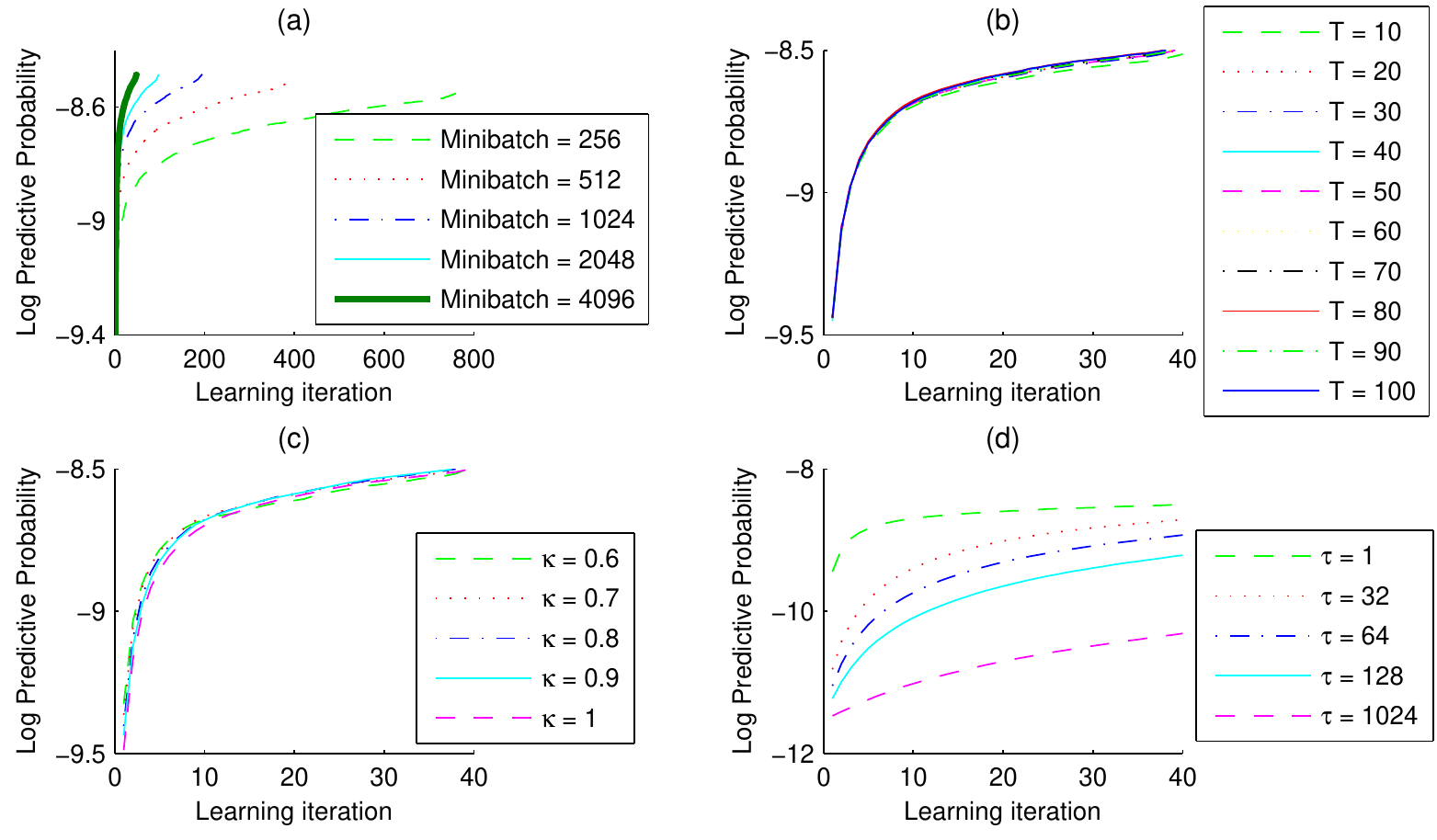}
\caption{Sensitivity of ML-OPE when changing its parameters. (a) Change the minibatch size when fixed $\{ \kappa=0.9, \tau=1, T=50\}$. (b) Change the number $T$ of iterations for OPE when fixed $\{ \kappa=0.9, \tau=1\}$. (c) Change the forgetting rate $\kappa$ when fixed $\{ \tau=1, T=50\}$. (d) Change $\tau$ when fixed $\{ \kappa=0.9, T=50\}$. The minibatch size in the cases of (b), (c), (d) is 5000. All of these experiments were done on New York Times, with $K=100$ topics.}
\label{fig:ML-OPE-sensitivity}
\end{figure}

We now consider the effects of the parameters on the performance of our new learning methods. The parameters include: the forgetting rate $\kappa$, $\tau$, the number $T$ of iterations for OPE, and the minibatch size. Inappropriate choices of those parameters might affect significantly the performance. To see the effect of a parameter, we changed its values in a finite set, but fixed the other parameters. We took ML-OPE into consideration, and results of our experiments are depicted in Figure~\ref{fig:ML-OPE-sensitivity}.

We observe that $\kappa$ and $T$ did not significantly affect the performance of ML-OPE. These behaviors of ML-OPE are interesting and beneficial in practice. Indeed, we do not have to consider much about the effect of the forgetting rate $\kappa$ and thus no expensive model selection is necessary. Figure~\ref{fig:ML-OPE-sensitivity}(b) reveals a much more interesting behavior of OPE. One easily observes that more iterations in OPE did not necessarily help the performance of ML-OPE. Just $T=20$ iterations for OPE resulted in a comparable predictiveness level as $T=100$. It suggests that OPE converges very fast in practice, and that $T=20$ might be enough for practical employments of OPE. This behavior is really beneficial in practice, especially for massive data or streaming data. 

$\tau$ and minibatch size did affect ML-OPE significantly.  Similar with the observation by \cite{Hoffman2013SVI} for SVI, we find that ML-OPE performed consistently better as the minibatch size increased. It can reach to a very high predictiveness level with a fast rate. In contrast, ML-OPE performed worse as $\tau$ increased. The method performed  best at $\tau=1$. It is worth noting that the dependence between the performance of ML-OPE and \{$\tau$, minibatch size\} is monotonic. Such a behavior enables us to easily choose a good setting for the parameters of ML-OPE in practice.

\subsection{Speed of inference methods}

\begin{figure}[tp]
\centering
\includegraphics[width=\textwidth]{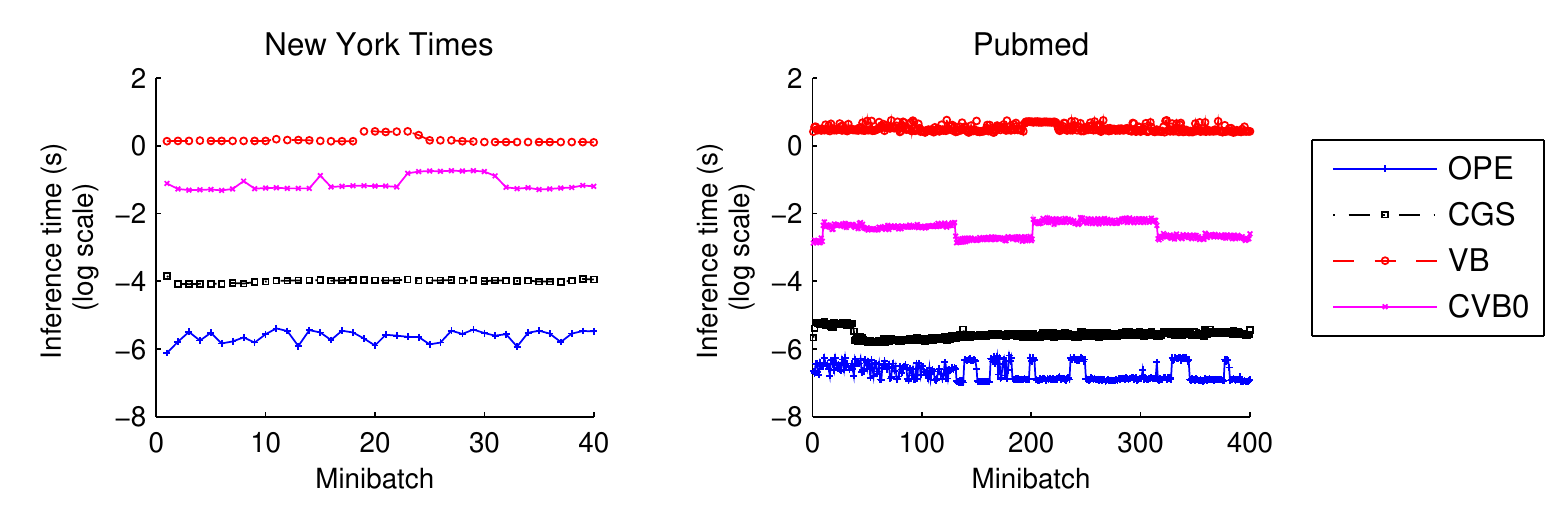}
\caption{Average time to do inference for a doccument as the number of minibatches increases. Lower is faster. Note that OPE often performs many times faster than CGS, and hundreds times faster than VB.}
\label{fig-OPE-inf-time}
\end{figure}

Next we investigate the speed of inference. We took VB, CVB0, CGS, and OPE into consideration. For all of these methods, we compute the average time to do inference for a document at every minibatch when learning LDA. Figure \ref{fig-OPE-inf-time} depicts the results. We find that among 4 inference methods, OPE consumed a modest amount of time, while CGS needed slightly more time. VB needed intensive time to do inference. The main reasons are that it requires various evaluations of expensive functions (e.g., log, exp, digamma), and that it needs to check convergence, which in our observation was often very expensive. Due to maintainance/update of many statistics which associate with each token in a document (see Table \ref{table 1: theoretical comparison}), CVB0 also consumed significant time. Note further that VB and CVB0 do not have any guarantee of  convergence rate. Hence in practice VB and CVB0 might converge slowly.

Figure \ref{fig-OPE-inf-time} suggests that OPE can perform fastest, compared with existing inference methods. Our investigation in the previous subsection demonstrates that OPE can find very good solutions for the posterior estimation problem. Those observations suggests that OPE is a good candidate for posterior inference in various situations.

\subsection{Convergence and stability of OPE in practice}

Our last investigation is about whether or not OPE performs stably in practice. We have to consider this behavior as there are two probabilistic steps in OPE: initialization of $\mbf{\theta_1}$ and pick of $f_{t}$. To see  stability, we took 100 testing documents  from New York Times to do inference given the 100-topic LDA model previously learned by ML-OPE. For each document, we did 10 random runs for OPE, saved the objective values of the last iterates, and then computed the standard deviation of the objective values. 

Stability of OPE is assessed via the standard deviation of the objective values. The smaller, the more stable. Figure \ref{fig-OPE-stability} shows the histogram of the standard deviations computed from 100 functions. Each $T$ corresponds to a choice of the number of iterations for OPE.

\begin{figure}[tp]
\centering
\includegraphics[width=\textwidth]{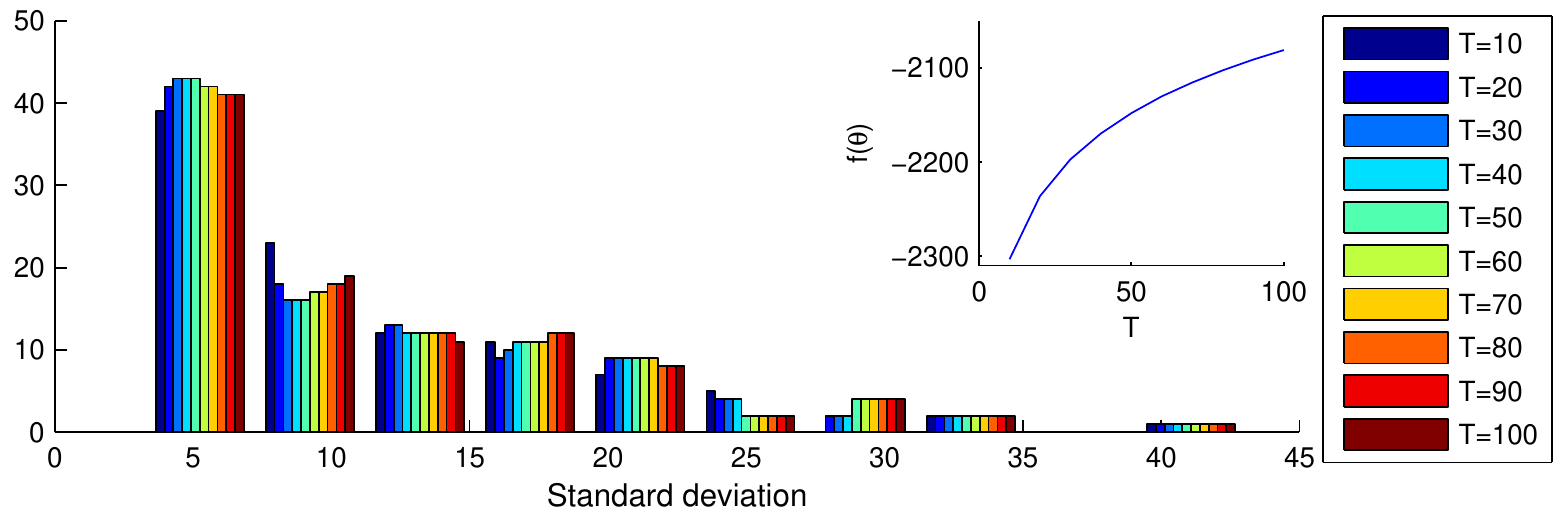}
\caption{Convergence and stability of OPE in practice. The top-right corner shows convergence of OPE as allowing more iterations. The stability of OPE is measured by the standard deviation of the objective values ($f(\mbf{\theta})$) from 10 random runs. The histogram (bar) for each setting of the number $T$ of iterations is computed from  100 different functions. This experiment was done on 100 documents of New York Times.}
\label{fig-OPE-stability}
\end{figure}

Observing Figure \ref{fig-OPE-stability}, we find that the standard deviation is small ($\le 20$) for a large amount of functions. Comparing with the mean value of $f(\mbf{\theta})$ which often belonged to $[-2300, -2000]$, the deviation is very small in magnitude. This suggests that for each function, the objective values returned by OPE from 10 runs seem not to  significantly differ from each other. In other words, OPE behaved very stable in our observation.

Figure \ref{fig-OPE-stability} also suggests that OPE converges very fast. The growth of $f(\mbf{\theta})$ seems to be close to a linear function. This agrees well with Theorem \ref{the9} on convergence rate of OPE.

\section{Conclusion} \label{sec-Conclusion}
We have discussed how posterior inference for individual texts in topic models can be done efficiently. Our novel algorithm (OPE) is the first one which has a theorerical guarantee on quality and fast convergence rate. In practice, OPE can do inference very fast, and can be easily extended to a wide range of contexts including MAP estimation and non-convex optimization. By exploiting OPE carefully, we have arrived at new efficient methods for learning LDA from data streams or large corpora: ML-OPE, Online-OPE, and Streaming-OPE. Among those, ML-OPE and Online-OPE can reach state-of-the-art performance at a high speed. Furthermore, Online-OPE surpasses all existing methods in terms of predictiveness, and works well with short text. As a result, they are good candidates to help us deal with text streams and big data. The code of those methods are available at \url{http://github.com/Khoat/OPE/}.

\acks{This research is funded by Vietnam National Foundation for Science and Technology Development (NAFOSTED) under Grant Number 102.05-2014.28 and by the Air Force Office of Scientific Research (AFOSR), Asian Office of Aerospace Research \& Development (AOARD), and US Army International Technology Center, Pacific (ITC-PAC) under Award Number FA2386-15-1-4011.}

\bibliography{../topic-models-all,../other-all,../my-papers}

\begin{thebibliography}{38}
\providecommand{\natexlab}[1]{#1}
\providecommand{\url}[1]{\texttt{#1}}
\expandafter\ifx\csname urlstyle\endcsname\relax
  \providecommand{\doi}[1]{doi: #1}\else
  \providecommand{\doi}{doi: \begingroup \urlstyle{rm}\Url}\fi

\bibitem[Aletras and Stevenson(2013)]{Aletras2013evaluating}
Nikolaos Aletras and Mark Stevenson.
\newblock Evaluating topic coherence using distributional semantics.
\newblock In \emph{Proceedings of the 10th International Conference on
  Computational Semantics}, pages 13--22, 2013.

\bibitem[Arora et~al.(2016)Arora, Ge, Koehler, Ma, and Moitra]{Arora+2016infer}
Sanjeev Arora, Rong Ge, Frederic Koehler, Tengyu Ma, and Ankur Moitra.
\newblock Provable algorithms for inference in topic models.
\newblock In \emph{ICML}, Journal of Machine Learning Research: W\&CP, 2016.

\bibitem[Asuncion et~al.(2009)Asuncion, Welling, Smyth, and
  Teh]{Asuncion+2009smoothing}
A.~Asuncion, M.~Welling, P.~Smyth, and Y.W. Teh.
\newblock On smoothing and inference for topic models.
\newblock In \emph{Proceedings of the Twenty-Fifth Conference on Uncertainty in
  Artificial Intelligence}, pages 27--34, 2009.

\bibitem[Blei(2012)]{Blei2012introduction}
David~M Blei.
\newblock Probabilistic topic models.
\newblock \emph{Communications of the ACM}, 55\penalty0 (4):\penalty0 77--84,
  2012.

\bibitem[Blei et~al.(2003)Blei, Ng, and Jordan]{BNJ03}
David~M. Blei, Andrew~Y. Ng, and Michael~I. Jordan.
\newblock Latent dirichlet allocation.
\newblock \emph{Journal of Machine Learning Research}, 3\penalty0 (3):\penalty0
  993--1022, 2003.

\bibitem[Bottou(1998)]{Bottou1998stochastic}
L{\'e}on Bottou.
\newblock Online learning in neural networks.
\newblock chapter Online Learning and Stochastic Approximations, pages 9--42.
  Cambridge University Press, 1998.

\bibitem[Bouma(2009)]{Bouma2009NPMI}
Gerlof Bouma.
\newblock Normalized (pointwise) mutual information in collocation extraction.
\newblock \emph{Proceedings of GSCL}, pages 31--40, 2009.

\bibitem[Broderick et~al.(2013)Broderick, Boyd, Wibisono, Wilson, and
  Jordan]{Broderick2013streaming}
Tamara Broderick, Nicholas Boyd, Andre Wibisono, Ashia~C Wilson, and Michael
  Jordan.
\newblock Streaming variational bayes.
\newblock In \emph{Advances in Neural Information Processing Systems}, pages
  1727--1735, 2013.

\bibitem[Clarkson(2010)]{Clarkson2010}
Kenneth~L. Clarkson.
\newblock Coresets, sparse greedy approximation, and the frank-wolfe algorithm.
\newblock \emph{ACM Trans. Algorithms}, 6:\penalty0 63:1--63:30, 2010.
\newblock ISSN 1549-6325.
\newblock \doi{http://doi.acm.org/10.1145/1824777.1824783}.
\newblock URL \url{http://doi.acm.org/10.1145/1824777.1824783}.

\bibitem[Dai et~al.(2016)Dai, He, Dai, and Song]{Dai2016pmd}
Bo~Dai, Niao He, Hanjun Dai, and Le~Song.
\newblock Provable bayesian inference via particle mirror descent.
\newblock In \emph{Proceedings of the 19th International Conference on
  Artificial Intelligence and Statistics}, pages 985--994, 2016.

\bibitem[Foulds et~al.(2013)Foulds, Boyles, DuBois, Smyth, and
  Welling]{Foulds2013stochastic}
James Foulds, Levi Boyles, Christopher DuBois, Padhraic Smyth, and Max Welling.
\newblock Stochastic collapsed variational bayesian inference for latent
  dirichlet allocation.
\newblock In \emph{Proceedings of the 19th ACM SIGKDD International Conference
  on Knowledge Discovery and Data Mining}, pages 446--454. ACM, 2013.

\bibitem[Gao et~al.(2015)Gao, Sun, Wang, Liu, Yan, and Zeng]{GaoSWYZ15}
Yang Gao, Zhenlong Sun, Yi~Wang, Xiaosheng Liu, Jianfeng Yan, and Jia Zeng.
\newblock A comparative study on parallel lda algorithms in mapreduce
  framework.
\newblock In \emph{Advances in Knowledge Discovery and Data Mining}, volume
  9078 of \emph{LNCS}, pages 675--689. Springer, 2015.

\bibitem[Gerrish and Blei(2012)]{GerrishB2012vote}
Sean Gerrish and David Blei.
\newblock How they vote: Issue-adjusted models of legislative behavior.
\newblock In \emph{Advances in Neural Information Processing Systems},
  volume~25, pages 2762--2770, 2012.

\bibitem[Ghadimi and Lan(2013)]{Ghadimi2013stochasticNonconvex}
Saeed Ghadimi and Guanghui Lan.
\newblock Stochastic first-and zeroth-order methods for nonconvex stochastic
  programming.
\newblock \emph{SIAM Journal on Optimization}, 23\penalty0 (4):\penalty0
  2341--2368, 2013.

\bibitem[Griffiths and Steyvers(2004)]{GriffithsS2004}
T.L. Griffiths and M.~Steyvers.
\newblock Finding scientific topics.
\newblock \emph{Proceedings of the National Academy of Sciences of the United
  States of America}, 101\penalty0 (Suppl 1):\penalty0 5228, 2004.

\bibitem[Grimmer(2010)]{Grimmer2010Political}
Justin Grimmer.
\newblock A bayesian hierarchical topic model for political texts: Measuring
  expressed agendas in senate press releases.
\newblock \emph{Political Analysis}, 18\penalty0 (1):\penalty0 1--35, 2010.
\newblock \doi{10.1093/pan/mpp034}.
\newblock URL \url{http://pan.oxfordjournals.org/content/18/1/1.abstract}.

\bibitem[Hazan and Kale(2012)]{Hazan2012OFW}
Elad Hazan and Satyen Kale.
\newblock Projection-free online learning.
\newblock In \emph{Proceedings of the 29th Annual International Conference on
  Machine Learning (ICML)}, 2012.

\bibitem[Hoffman et~al.(2013)Hoffman, Blei, Wang, and Paisley]{Hoffman2013SVI}
Matthew~D Hoffman, David~M Blei, Chong Wang, and John Paisley.
\newblock Stochastic variational inference.
\newblock \emph{The Journal of Machine Learning Research}, 14\penalty0
  (1):\penalty0 1303--1347, 2013.

\bibitem[Lau et~al.(2014)Lau, Newman, and Baldwin]{Lau2014npmi}
Jey~Han Lau, David Newman, and Timothy Baldwin.
\newblock Machine reading tea leaves: Automatically evaluating topic coherence
  and topic model quality.
\newblock In \emph{Proceedings of the Association for Computational
  Linguistics}, pages 530--539, 2014.

\bibitem[Liu et~al.(2010)Liu, Liu, Tsykin, Goodall, Green, Zhu, Kim, and
  Li]{LiuLT10miRNA}
B.~Liu, L.~Liu, A.~Tsykin, G.J. Goodall, J.E. Green, M.~Zhu, C.H. Kim, and
  J.~Li.
\newblock Identifying functional mirna--mrna regulatory modules with
  correspondence latent dirichlet allocation.
\newblock \emph{Bioinformatics}, 26\penalty0 (24):\penalty0 3105, 2010.

\bibitem[Mai et~al.(2016)Mai, Mai, Nguyen, Van~Linh, and Than]{Mai2016hdp}
Khai Mai, Sang Mai, Anh Nguyen, Ngo Van~Linh, and Khoat Than.
\newblock Enabling hierarchical dirichlet processes to work better for short
  texts at large scale.
\newblock In \emph{Advances in Knowledge Discovery and Data Mining}, volume
  9652 of \emph{LNCS}, pages 431--442. 2016.

\bibitem[Mairal(2013)]{Mairal2013stochasticNonconvex}
Julien Mairal.
\newblock Stochastic majorization-minimization algorithms for large-scale
  optimization.
\newblock In \emph{Advances in Neural Information Processing Systems}, pages
  2283--2291, 2013.

\bibitem[McInerney et~al.(2015)McInerney, Ranganath, and
  Blei]{McInerney2015population}
James McInerney, Rajesh Ranganath, and David Blei.
\newblock The population posterior and bayesian modeling on streams.
\newblock In \emph{Advances in Neural Information Processing Systems}, pages
  1153--1161, 2015.

\bibitem[Mimno(2012)]{Mimno2012historiography}
David Mimno.
\newblock Computational historiography: Data mining in a century of classics
  journals.
\newblock \emph{Journal on Computing and Cultural Heritage}, 5\penalty0
  (1):\penalty0 3, 2012.

\bibitem[Mimno et~al.(2012)Mimno, Hoffman, and Blei]{MimnoHB12}
David Mimno, Matthew~D. Hoffman, and David~M. Blei.
\newblock Sparse stochastic inference for latent dirichlet allocation.
\newblock In \emph{Proceedings of the 29th Annual International Conference on
  Machine Learning}, 2012.

\bibitem[Patterson and Teh(2013)]{Patterson2013stochastic}
Sam Patterson and Yee~Whye Teh.
\newblock Stochastic gradient riemannian langevin dynamics on the probability
  simplex.
\newblock In \emph{Advances in Neural Information Processing Systems}, pages
  3102--3110, 2013.

\bibitem[Pritchard et~al.(2000)Pritchard, Stephens, and
  Donnelly]{PritchardSD2000population}
Jonathan~K Pritchard, Matthew Stephens, and Peter Donnelly.
\newblock Inference of population structure using multilocus genotype data.
\newblock \emph{Genetics}, 155\penalty0 (2):\penalty0 945--959, 2000.

\bibitem[Sato and Nakagawa(2015)]{SatoN2015SCVB0}
Issei Sato and Hiroshi Nakagawa.
\newblock Stochastic divergence minimization for online collapsed variational
  bayes zero inference of latent dirichlet allocation.
\newblock In \emph{Proceedings of the 21th ACM SIGKDD International Conference
  on Knowledge Discovery and Data Mining}, pages 1035--1044. ACM, 2015.

\bibitem[Schwartz et~al.(2013)Schwartz, Eichstaedt, Dziurzynski, Kern,
  Seligman, Ungar, Blanco, Kosinski, and Stillwell]{Schwartz+2013Personality}
H~Andrew Schwartz, Johannes~C Eichstaedt, Lukasz Dziurzynski, Margaret~L Kern,
  Martin~EP Seligman, Lyle~H Ungar, Eduardo Blanco, Michal Kosinski, and David
  Stillwell.
\newblock Toward personality insights from language exploration in social
  media.
\newblock In \emph{AAAI Spring Symposium Series}, 2013.

\bibitem[Simsekli et~al.(2016)Simsekli, Badeau, Richard, and
  Cemgil]{Simsekli2016stochastic}
Umut Simsekli, Roland Badeau, Ga{\"e}l Richard, and Taylan Cemgil.
\newblock Stochastic quasi-newton langevin monte carlo.
\newblock In \emph{Proceedings of the 33th International Conference on Machine
  Learning (ICML)}, 2016.

\bibitem[Sontag and Roy(2011)]{SontagR11}
David Sontag and Daniel~M. Roy.
\newblock Complexity of inference in latent dirichlet allocation.
\newblock In \emph{Advances in Neural Information Processing Systems (NIPS)},
  2011.

\bibitem[Tang et~al.(2014)Tang, Meng, Nguyen, Mei, and
  Zhang]{Tang2014understandingLDA}
Jian Tang, Zhaoshi Meng, Xuanlong Nguyen, Qiaozhu Mei, and Ming Zhang.
\newblock Understanding the limiting factors of topic modeling via posterior
  contraction analysis.
\newblock In \emph{Proceedings of The 31st International Conference on Machine
  Learning}, pages 190--198, 2014.

\bibitem[Teh et~al.(2016)Teh, Thiery, and Vollmer]{Teh2016consistencySGLD}
Yee~Whye Teh, Alexandre~H Thiery, and Sebastian~J Vollmer.
\newblock Consistency and fluctuations for stochastic gradient langevin
  dynamics.
\newblock \emph{Journal of Machine Learning Research}, 17\penalty0
  (7):\penalty0 1--33, 2016.

\bibitem[Teh et~al.(2007)Teh, Newman, and Welling]{TehNW2007collapsed}
Y.W. Teh, D.~Newman, and M.~Welling.
\newblock A collapsed variational bayesian inference algorithm for latent
  dirichlet allocation.
\newblock In \emph{Advances in Neural Information Processing Systems},
  volume~19, page 1353, 2007.

\bibitem[Than and Doan(2014)]{ThanD14dolda}
Khoat Than and Tung Doan.
\newblock Dual online inference for latent dirichlet allocation.
\newblock In \emph{Proceedings of the 6th Asian Conference on Machine Learning
  (ACML)}, volume~39 of \emph{Journal of Machine Learning Research: W\&CP},
  pages 80--95, 2014.

\bibitem[Than and Ho(2015)]{ThanH15sparsity}
Khoat Than and Tu~Bao Ho.
\newblock Inference in topic models: sparsity and trade-off.
\newblock \emph{arXiv preprint arXiv:1512.03300}, 2015.
\newblock URL \url{http://arxiv.org/abs/1512.03300}.

\bibitem[Than and Ho(2012)]{ThanH2012fstm}
Khoat Than and Tu~Bao Ho.
\newblock Fully sparse topic models.
\newblock In Peter Flach, Tijl De~Bie, and Nello Cristianini, editors,
  \emph{Machine Learning and Knowledge Discovery in Databases}, volume 7523 of
  \emph{Lecture Notes in Computer Science}, pages 490--505. Springer, 2012.

\bibitem[Yuille and Rangarajan(2003)]{Yuille2003cccp}
Alan~L Yuille and Anand Rangarajan.
\newblock The concave-convex procedure.
\newblock \emph{Neural computation}, 15\penalty0 (4):\penalty0 915--936, 2003.

\end{thebibliography}

\appendix 

\section{Predictive Probability}
\label{appendix--perp}

Predictive Probability shows the predictiveness and generalization of a model $\mathcal{M}$ on new data. We followed the procedure in \citep{Hoffman2013SVI} to compute this quantity. For each document in a testing dataset, we divided randomly into two disjoint parts $\mbf{w}_{obs}$ and $\mbf{w}_{ho}$ with a ratio of 70:30. We next did inference for $\mbf{w}_{obs}$ to get an estimate of $\mathbb{E}(\mbf{\theta}^{obs})$. Then we approximated the predictive probability as 
\[
\Pr(\mbf{w}_{ho} | \mbf{w}_{obs}, \mathcal{M}) \approx \prod_{w \in \mbf{w}_{ho} } \sum_{k=1}^{K} \mathbb{E}(\mbf{\theta}^{obs}_k) \mathbb{E}(\mbf{\beta}_{kw}),
\]
\[
\text{Log Predictive Probability} = \frac{\log \Pr(\mbf{w}_{ho} | \mbf{w}_{obs}, \mathcal{M})}  {|\mbf{w}_{ho}|},
\]
where $\mathcal{M}$ is the model to be measured. We estimated $\mathbb{E}(\mbf{\beta}_k) \propto \mbf{\lambda}_k$ for the learning methods which maintain a variational distribution ($\mbf{\lambda}$) over topics. Log Predictive Probability was averaged from 5 random splits, each was on 1000 documents.

\section{NPMI}

\emph{NPMI}  \citep{Aletras2013evaluating,Bouma2009NPMI} is the measure to help us see the coherence or semantic quality of individual topics. According to \cite{Lau2014npmi}, NPMI  agrees well with human evaluation on interpretability of topic models. For each topic $t$, we take the set $\{w_1, w_2, ..., w_n\}$ of top $n$ terms with highest probabilities. We then computed

\[
NPMI(t) = \frac{2}{n(n-1)} \sum_{j=2}^{n} \sum_{i=1}^{j-1} \frac{\log \frac{P(w_j, w_i)}{P(w_j) P(w_i)}}{- \log P(w_j, w_i)},
\]
where $P(w_i, w_j)$ is the probability that  terms $w_i$ and $w_j$ appear together in a document. We estimated those probabilities from the training data. In our experiments, we chose top $n=10$ terms for each topic.

Overall, NPMI of a model with $K$ topics is averaged as:
\[
NPMI = \frac{1}{K} \sum_{t=1}^{K} NPMI(t).
\]

\end{document}